\documentclass{article}

\usepackage[final]{corl_2019} 
\usepackage{microtype}
\usepackage{graphicx}
\usepackage{mathtools}
\usepackage{siunitx}
\usepackage[font=small,tableposition=top]{caption}
\usepackage{booktabs} 
\usepackage{amssymb}
\usepackage{amsmath,amsthm}
\usepackage{mathtools}
\usepackage[usenames,dvipsnames]{xcolor}
\usepackage{subcaption}
\usepackage{algorithm}
\usepackage[noend]{algpseudocode}
\usepackage[utf8]{inputenc}
\usepackage[english]{babel}
\usepackage[inline]{enumitem}

\newtheorem{lemma}{Lemma}[section]

\usepackage{dsfont}
\theoremstyle{definition}
\theoremstyle{assumption}

\newtheorem{assumption}{Assumption}[section]
\usepackage{hyperref}
\usepackage{cleveref}
\usepackage{changepage}
\usepackage{algorithmicx} 

\usepackage{titlesec}
\titlespacing\section{0pt}{0pt plus 2pt minus 2pt}{0pt plus 2pt minus 2pt}
\titlespacing\subsection{0pt}{3pt plus 4pt minus 2pt}{0pt plus 2pt minus 2pt}
\titlespacing\subsubsection{0pt}{3pt plus 4pt minus 2pt}{0pt plus 2pt minus 2pt}

\newcommand{\todo}[1]{}
\renewcommand{\todo}[1]{{\color{red} TODO: {#1}}}

\newcommand{\incmtt}[1]{{\fontfamily{cmtt}\selectfont{#1}}} 
\definecolor{darkblue}{rgb}{0.0,0.0,0.7} 
\renewcommand\algorithmiccomment[1]{%
  \textcolor{darkblue}{\scriptsize \incmtt{\hfill\#\,{#1}}}%
}
\algnewcommand{\LineComment}[1]{\textcolor{darkblue}{\scriptsize{\incmtt{\#\ #1}}}}

\title{Untangling Dense Knots \\ by Learning Task-Relevant Keypoints}

\author{
 Jennifer Grannen*$^{1}$, Priya Sundaresan*$^{1}$, Brijen Thananjeyan$^{1}$, Jeffrey Ichnowski$^{1}$, \\
 \textbf{Ashwin Balakrishna$^{1}$, Minho Hwang$^1$, Vainavi Viswanath$^{1}$, Michael Laskey$^{2}$,} \\
 \textbf{Joseph E. Gonzalez$^{1}$, Ken Goldberg$^{1}$} \\
  \scriptsize{* equal contribution} \\
  \texttt{jenngrannen@berkeley.edu}, \texttt{priyasundaresan@berkeley.edu}
}

\makeatletter
\def\thanks#1{\protected@xdef\@thanks{\@thanks
        \protect\footnotetext{#1}}}
\makeatother
\thanks{$^{1}$University of California, Berkeley. $^{2}$Toyota Research Institute. }
\begin{document}
\newcommand\smallO{
  \mathchoice
    {{\scriptstyle\mathcal{O}}}
    {{\scriptstyle\mathcal{O}}}
    {{\scriptscriptstyle\mathcal{O}}}
    {\scalebox{.6}{$\scriptscriptstyle\mathcal{O}$}}
  }


\newcommand{\brijen}[1]{\textcolor{blue}{(#1 --Brijen)}}
\newcommand{\ashwin}[1]{\textcolor{magenta}{(#1 --Ashwin)}}
\newcommand{\jeffi}[1]{{\color{blue}(#1 --JI)}}
\newcommand{\jenn}[1]{{\color{violet}(#1 --Jenn)}}

\maketitle
\vspace{-0.7cm}
\begin{abstract}
Untangling ropes, wires, and cables is a challenging task for robots due to the high-dimensional configuration space, visual homogeneity, self-occlusions, and complex dynamics. We consider dense (tight) knots that lack space between self-intersections and present an iterative approach that uses learned geometric structure in configurations. We instantiate this into an algorithm, HULK: Hierarchical Untangling from Learned Keypoints, which combines learning-based perception with a geometric planner into a policy that guides a bilateral robot to untangle knots. To evaluate the policy, we perform experiments both in a novel simulation environment modelling cables with varied knot types and textures and in a physical system using the da Vinci surgical robot. We find that HULK is able to untangle cables with dense figure-eight and overhand knots and generalize to varied textures and appearances. We compare two variants of HULK to three baselines and observe that HULK achieves 43.3\% higher success rates on a physical system compared to the next best baseline. HULK successfully untangles a cable from a dense initial configuration containing up to two overhand and figure-eight knots in 97.9\% of 378 simulation experiments with an average of 12.1 actions per trial. In physical experiments, HULK achieves 61.7\% untangling success, averaging 8.48 actions per trial. Supplementary material, code, and videos can be found at \url{https://tinyurl.com/y3a88ycu}.
\end{abstract}


\keywords{Deformable Manipulation, Computer Vision} 

\section{Introduction}
Untangling ropes, wires, cables, strings, and hoses has a wide range of applications in surgery~\cite{mayer2008system, SAVED, van2010superhuman}, manufacturing~\cite{yamakawa2007one}, and households~\cite{sanchez2018robotic}. We broadly refer to this class of 1D deformable objects as ``cables". Manipulating cables poses two primary challenges: (1) state estimation and (2) manipulation. State estimation is complicated by the high-dimensional configuration space, visual homogeneity of the cable, and self-occlusions present in knots. Manipulation is complicated by friction and tension in knots, stiffness, and the need for bilateral motions to achieve loosening and manage slack~\cite{schulman2016learning, sundaresan2020learning, lui2013tangled, nair2017combining}. Prior work on untangling cables has explored both classical~\cite{lui2013tangled, schulman2016learning} and learning-based perception methods~\cite{sundaresan2020learning} to estimate the state of loosely knotted cables, and used geometric algorithms or learning from demonstrations to perform manipulation~\cite{lui2013tangled, schulman2016learning, sundaresan2020learning, kurutach2018model}. However, to infer cable state, these methods often rely on accurate segmentation, which is challenging in densely knotted configurations that lack space between crossings of adjacent segments~\cite{wang2016tying}. To address these challenges, we develop a learning based algorithm which leverages recent advances in deep learning for perception to learn task relevant keypoints.

We present 2 algorithms. The first, Basic Reduction of Under-Crossing Entanglements (BRUCE), assumes a known graph structure for knotted cables and sequentially undoes one crossing at a time until the cable has no self-intersections using two manipulation primitives from knot theory~\cite{lui2013tangled}.
Using BRUCE as an algorithmic supervisor with full state information in simulation, the learning-based algorithm, HULK: Hierarchical Untangling from Learned Keypoints, builds on recent work in learned keypoint detection~\cite{papandreou2017towards} to use visual input to directly locate pin and pull points in densely knotted cables. While prior work in deformable object tracking estimates the full cable state (e.g., through dense object descriptors~\cite{sundaresan2020learning,florence2018dense,schmidt2016self}), full state estimation
is unnecessary for tasks such as untangling.

\begin{figure}[!htbp]
    \vspace{-0.2cm}
    \centering
    \includegraphics[width=0.98\linewidth]{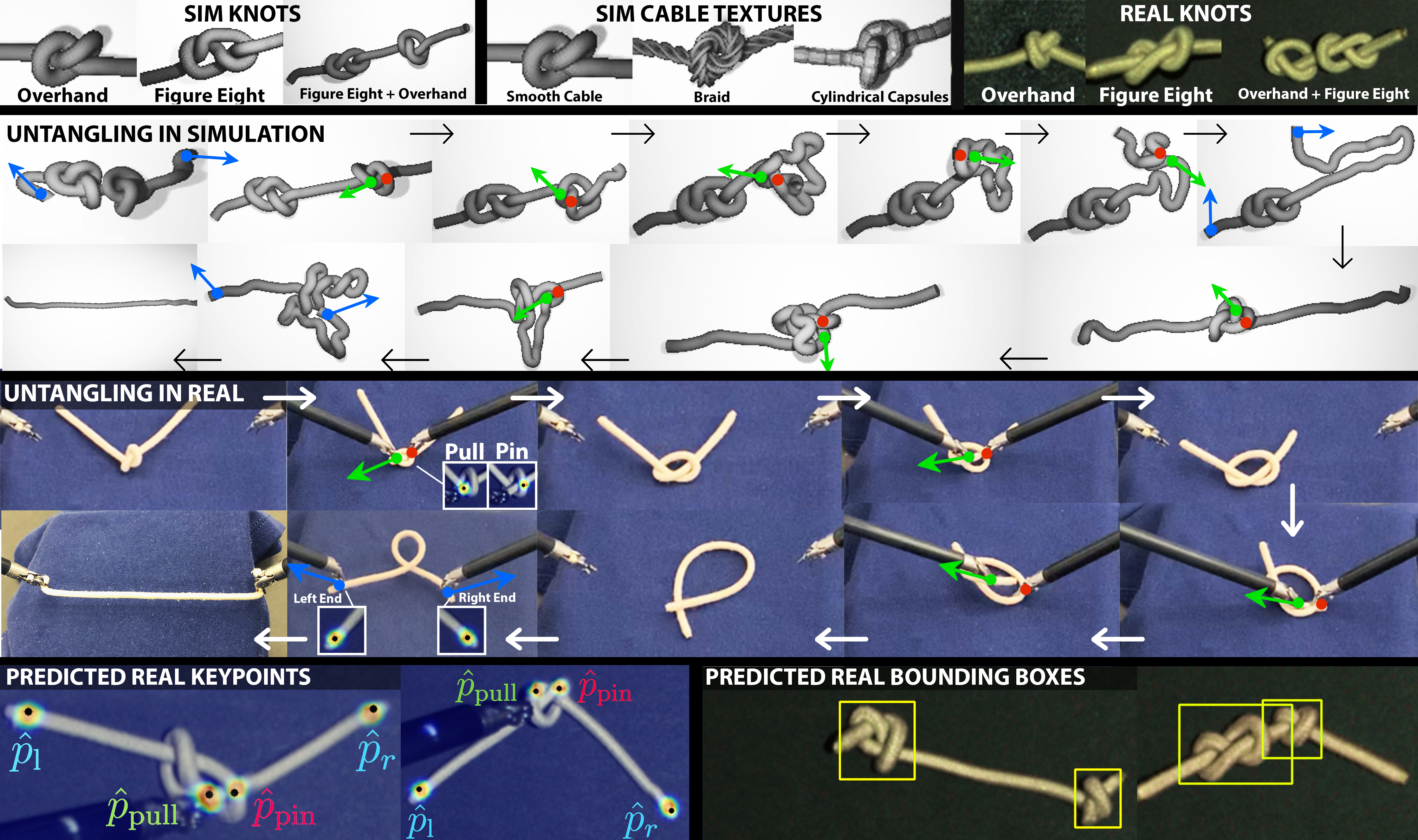}
    \caption{\textbf{HULK} learns to identify pin (red) and pull (green) keypoints from images to untangle dense figure-eight and overhand knots with generalization to various cable textures in simulation (first row). The second panel shows a sequence of untangling actions on a simulated cable with a figure-eight knot and an overhand knot. The third panel shows real actions taken by a da Vinci Research Kit robot to undo a dense overhand knot in an elastic hair tie cable. Straightening (Reidemeister) pick-place moves, shown in blue, pull the cable taut to remove occlusions that are not part of a knot, and crossing reduction (Node Deletion) moves (red pin and green pull) sequentially eliminate self-intersections.}
    \label{fig:simulator_fig}
    \vspace{-0.3cm} 
\end{figure}

This paper makes four contributions: (1) an open-access simulator in Blender for modelling the dynamics of cables; (2) BRUCE (Basic Reduction of Under-Crossing Entanglements), a geometric algorithm for untangling knots, (3) HULK (Hierarchical Untangling from Learned Keypoints): a learning-based algorithm for untangling densely knotted cables from visual input, and (4) simulation and physical experiments comparing HULK with several baseline policies with alternative perception pipelines or access to ground truth. These experiments suggest that HULK can achieve untangling with fewer actions than baselines that have access to more state information, and that HULK can generalize to cables with novel visual appearance and knot configurations.

\section{Related Work}

Manipulation of cables using analytic \cite{lui2013tangled, schulman2016learning, lee2015learning, she2019cable}, learned \cite{yan2020learning, nair2017combining}, and hybrid \cite{sundaresan2020learning, wang2019learning} approaches has received considerable recent interest, but to the best of our knowledge, handling dense knots remains an open problem. Effective cable manipulation also often requires some degree of state estimation from images, such as through nonrigid registration~\cite{chui2003new} or deep-learning approaches for global object correspondence~\cite{florence2018dense, schmidt2016self, sundaresan2020learning, kulkarni2019canonical}. We build on the work of~\citet{lui2013tangled}, which achieves 76.9\% success on physical experiments of untangling loosely knotted cables. In their pioneering 2013 paper, the authors represent a cable as a graph approximated using RGB-D perception and feature engineering~\cite{lui2013tangled}. 
They plan untangling actions by manipulating nodes and edges to minimize crossings in the graph, taking into account heuristics such as measurements of slack and empty space. In contrast, we present a geometrically motivated algorithm that circumvents full state estimation coupled with a learning-based policy which generalizes to densely knotted cables. 

Several works have also demonstrated robot knot tying and cable arrangement given human demonstrations. \citet{sundaresan2020learning} apply a geometric approach to both tasks by learning a dense object descriptor mapping~\cite{florence2018dense, schmidt2016self} in simulation, enabling deformable correspondence and tracking. This method learns global correspondences, and provides a similar ordering of pixels to a graph structure. However, when cable configurations are complex, robust state estimation becomes challenging due to inadequate cable segmentation and overlapping segments (Fig.~\ref{fig:perception_fig}~B). In contrast, HULK focuses only on task-relevant keypoints, enabling fine visuomotor control. Other methods present algorithms for knot tying and cable manipulation but do not leverage geometric structure~\cite{nair2017combining, schulman2016learning, SAVED}, whereas we reason about the geometry of intersections to infer appropriate untangling actions. We also build on recent approaches to cable manipulation that decouple perception from control~\cite{sundaresan2020learning, wang2019learning, yan2020learning}. While these approaches are sample efficient and modular, they attempt to learn intermediate or explicit state representations unlike HULK, which estimates task specific keypoints.

Several researchers have focused on a larger class of 1D and 2D deformable objects such as fabrics and clothing. Deformable object manipulation using classical perception methods~\cite{lui2013tangled, chi2019occlusion,wakamatsu2006knotting} has proven effective, but requires hand-engineered features specific to the experimental setup and can fail to handle highly deformed configurations. To mitigate the perception and modelling challenges associated with deformable objects, several works have employed deep learning-based methods including reinforcement learning~\cite{thananjeyan2017multilateral, wu2019learning, matas2018sim}, imitation learning~\cite{pathak2018zero,seita2019deep, zhang2018deep}, video prediction~\cite{hoque2020visuospatial}, and other approaches~\cite{sundaresan2020learning,ganapathi2020learning,nair2017combining,schulman2016learning,lee2015learning}. However, accurate reward engineering and loss function specification remains challenging in these methods. End-to-end learning also can result in black-box policies that abstract away visual reasoning and lack interpretability. Instead, we separate visual reasoning from manipulation to learn task-specific geometry and generate visually interpretable manipulation plans. Similar strategies decoupling perception and manipulation have proven effective in a range of robotic manipulation tasks~\cite{sundaresan2020learning, ganapathi2020learning,  danielczuk2020x, florence2019self, zeng2018robotic}. 


\def\one{\mbox{1\hspace{-4.25pt}\fontsize{12}{14.4}\selectfont\textrm{1}}}
\section{Problem Statement}
\label{sec:ps}

Given an RGB image of a cable in a densely knotted initial configuration, a bilateral robot attempts to manipulate it with a sequence of pin and pull actions to achieve a fully untangled state with no crossings within a fixed number of steps. 
\vspace{-0.3cm}
\paragraph{Assumptions}
We define the manipulation workspace with a standard (x,y,z) coordinate frame and image observations $I \in \mathbb{R}^{1200 \times 1900 \times 3}$. We assume the cable is fully visible and can be distinguished from the workspace background with color thresholding. We assume that a bilateral robot with two grippers can pin and pull at cable points within the workspace. We define two points, $w_r$ and $w_l$, on the right and left ends of the untangling workspace to be used in action planning.
As shown in Fig.~\ref{fig:problem_statement_fig}, we make the following assumptions on the starting configuration of the cable:
(1) \emph{semi-planarity}: each crossing has at most two cable segments (four edges),
(2) \emph{knot structure}: overhand or figure-eight knots, and
(3) \emph{visible endpoints}: both endpoints of the cable are visible and we distinguish between the two as left and right. We assign the endpoints with the smaller and larger x-values to be the left and right endpoints, respectively, breaking ties arbitrarily.
We define a dense cable configuration to be one that lacks space between adjacent cable segments connecting any two occlusions. 

For BRUCE, we initially formulate the problem using an undirected graph $G = (V, E)$ to model the cable configuration~\cite{lui2013tangled}; HULK, however, only takes an image as input and does \emph{not} attempt to explicitly reconstruct $G$ or estimate its parameters. 

As illustrated in Fig.~\ref{fig:problem_statement_fig}, a linear graph can model any cable configuration. Each vertex $v \in V$ is an intersection (node) or endpoint; intermediate points between nodes and endpoints are not included in the graph since they do not participate in crossings. Each edge $e \in E$ is defined as a 2-vertex tuple $e = (u,v), u \in V, v \in V$. Under the semi-planar assumption, every vertex has degree 1 or 4, $|V| = N+2$, and $|E| = 2N + 1$, where $N$ is the number of nodes. We denote the leftmost degree-one vertex $v_{l}$, and the rightmost degree-one vertex (endpoint) $v_{r}$ breaking ties arbitrarily. It is possible to have multiple edges between vertices such as between nodes $2$ and $3$ in Fig.~\ref{fig:problem_statement_fig}. For every vertex $v$, we annotate the incident edges $e = (v, v’) \in E$ with $+1$ or $-1$ as follows (Fig.~\ref{fig:problem_statement_fig}). 
\begin{align}
\label{eq:annotations}
X(v, e) = 
\begin{cases} 
+1 & \text{if vertex $v$ is an endpoint or if $e$ crosses above the other edge at vertex $v$} \\ -1 & \text{otherwise} 
\end{cases}
\end{align}
When $|V| = 2$ the cable is in the desired state with two endpoints exposed and no intersections. 

\begin{figure}[!htbp]
    \centering
    \vspace{-0.3cm}
    \includegraphics[width=1.0\linewidth]{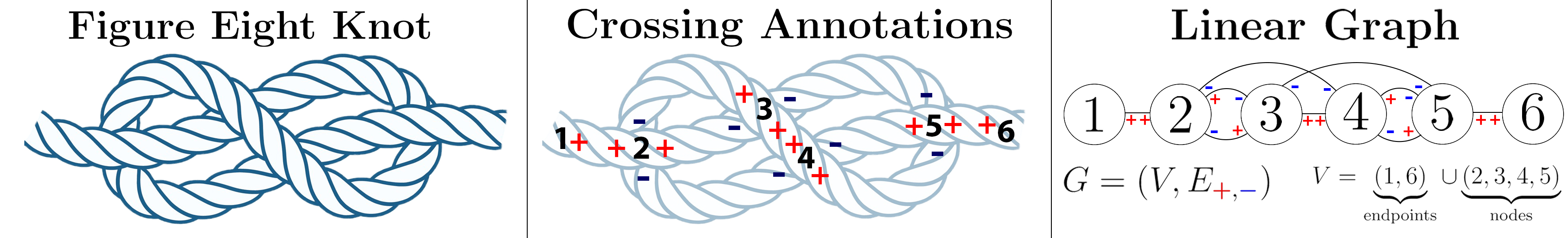}
    \caption{\textbf{Cable Graph Representation}: We consider the task of untangling cables with dense knots and varied textures. Each type of knot is parameterized by the locations of nodes (intersections) and endpoints and can be represented as a linear graph as in \citet{lui2013tangled} annotated with (+1, -1) according to Section \ref{sec:ps}.}
    \label{fig:problem_statement_fig}
    \vspace{-0.4cm}
\end{figure}

At each time $t$, we consider two simultaneously executed actions, one for each gripper: pin and pull, both with 4 variables in the global workspace coordinate frame:
\[ \mathbf{a_{t,l}} = (x_{t,l}, y_{t,l}, \Delta x_{t,l}, \Delta y_{t,l}) \hspace{2mm}\vert \hspace{2mm} \mathbf{a_{t,r}} = (x_{t,r}, y_{t,r}, \Delta x_{t,r}, \Delta y_{t,r})\]

As shown in Figure~\ref{fig:untangling_fig}, the right gripper performs a pin action $\mathbf{a_{t,l}}$ by grasping the cable at $(x_{t,l}, y_{t,l})$. The left gripper performs a pull action $\mathbf{a_{t,r}}$ by grasping the cable at $(x_{t,r}, y_{t,r})$, lifting by a fixed offset, and travelling by $(\Delta x_{t,r}, \Delta y_{t,r})$ before releasing. 

\section{BRUCE: Basic Reduction of Under-Crossing Entanglements}
\label{sec:alg}
BRUCE: Basic Reduction of Under-Crossing Entanglements is an algorithm for untangling from a semi-planar starting configuration, given the graph structure from Section~\ref{sec:ps}.

As in \citet{lui2013tangled}, we use pinning and pulling actions to execute two types of manipulation primitives: \textbf{Reidemeister moves} and \textbf{Node Deletion moves}. Reidemeister moves remove occlusions that are not part of a knot by pulling each end of the cable to opposite sides of the workspace at predefined points $w_r$ and $w_l$. Node Deletion moves remove a node in the graphical abstraction of the cable by pulling the endpoint corresponding to the incident edge labelled $-1$ from one side of the crossing to the other while pinning down the other incident edge labelled $+1$ to prevent the rest of the configuration from shifting. This action removes one node while leaving the other nodes in the cable's configuration unchanged, reducing $|V|$ by one and $|E|$ by two. The termination condition is when the cable's configuration contains no crossings, corresponding to a graph representation with $|V| = 2$ for the endpoints.

BRUCE iteratively undoes crossings starting from the right endpoint until no self occlusions remain. It starts by performing a Reidemeister move to remove crossings separate from the relevant knots and disambiguate the configuration. Next, it identifies the next \textbf{under-crossing}, $c$, that must be undone by selecting the first vertex adjacent to two successive $-$ annotated edges. It then performs one Node Deletion move on $c$ followed by a Reidemeister move. It repeats these two steps until there are no occlusions left in the cable. This procedure is illustrated in Fig.~\ref{fig:untangling_fig}.

\begin{figure}[!thbp]
    \centering
    \vspace{-0.2cm}
    \includegraphics[width=1.0\linewidth]{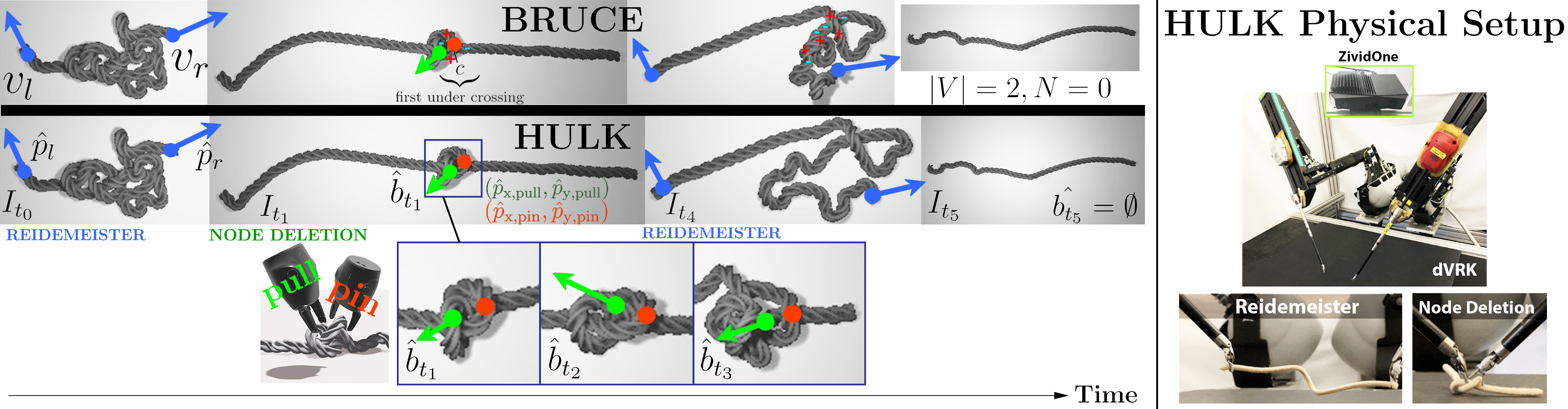}
    \caption{\textbf{BRUCE/HULK Untangling Algorithm Comparison} (left, from top to bottom): Given a semi-planar, densely knotted cable graph, BRUCE performs a Reidemeister move with endpoint vertices $v_l$ and $v_r$. Next, BRUCE locates the first under-crossing $c$ relative to $v_r$. While pinning the cable in place, a pull action is performed in a Node Deletion move until $c$ is removed. A final Reidemeister move is taken to restore the cable to a straightened configuration. Starting from time $t_0$, HULK follows a similar procedure, identifying keypoints $\hat{p}_l$, $\hat{p}_r$, $\hat{p}_{\mathrm{pin}}$, and $\hat{p}_\mathrm{{pull}}$ from image observations $I_{t_i}$ to perform Reidemeister and Node Deletion moves. HULK performs these moves in sequence for as long as a knot bounding box $\hat{b}_{t_i}$ is detected in the current observation. In HULK, slack management requires multiple pinning and pulling actions to be performed for each Node Deletion move, as illustrated in the third row of bounding box images. We implement the manipulation primitives for HULK on a da Vinci Research Kit robot \cite{chen2017software} with a mounted ZividOne overhead RGBD camera (right).} 
    \label{fig:untangling_fig}\vspace{-0.2cm}
\end{figure}


If each operation is executed correctly on the graphical representation, BRUCE is guaranteed to untangle any semi-planar configuration. 
Each iteration of BRUCE 
performs 
Node Deletion and Reidemeister operations that monotonically reduce the number of crossings until no crossings remain.

\section{HULK: Hierarchical Untangling from Learned Keypoints}
\label{sec:implementation}
HULK does not assume the graph structure is given and learns task-relevant features to infer actions to perform BRUCE's manipulation primitives on a loosely or densely-knotted cable.

\subsection{Simulation}
We develop a simulation environment for scalable training and evaluation of the learned policy subject to knot type and cable appearance variations. We use Blender 2.8~\cite{lallemand1998blender}, a graphics and animation suite that supports physics simulations, to create a cable simulation environment. Compared with other deformable simulators (e.g., Mujoco~\cite{todorov2012mujoco} and PyBullet~\cite{coumans2016pybullet}), Blender provides greater flexibility over the appearance of the cable for domain randomization. We construct a mass-spring model of a cable consisting of 50 rigid-body cylindrical meshes linked by springs following prior work~\cite{seita2019deep, yan2020self}. The simulator supports two textures for the cable model, (1) a smooth texture consistent with hoses, cables, and tubing; and (2) a braided texture modelled after twine and nylon rope (Fig.~\ref{fig:simulator_fig}). We hand-code one knot-tying trajectory per knot type, and add random noise to each trajectory to generate initially perturbed dense configurations. 

\subsection{Dataset Generation}

\textbf{Synthetic Data}: To generate synthetic training data, we produce a variety of dense initial configurations by adding random noise into predefined knot-tying trajectories. We use the ordered set of cylinder positions, queried from Blender's Python API, and ray tracing to detect crossings and infer the graph. Given the reconstructed graphs, we execute BRUCE and record overhead synthetic RGB renderings and ground-truth annotations with which to train HULK.
\\\textbf{Real Data}: HULK is learned entirely from real data for deployment on a physical robot. We hand-label a set of real cable images with crossing, endpoint, and bounding box annotations, and use data augmentation techniques discussed in Appendix \ref{sec:experimental_details} to amplify the dataset size by 20X before training.


\subsection{HULK Implementation}
HULK uses learned keypoints to model only the task-specific portion of the cable. The algorithm is partitioned across: (1) planning actions to execute Reidemeister and Node Deletion moves and (2) perceiving the necessary geometric features for these actions. 
We introduce two versions of HULK that differ only in perception: (1) \textbf{HULK-G}, which operates \emph{exclusively} on the \textbf{global} image of a cable, and (2) \textbf{HULK-L}, which operates on \textbf{local} knot crops of a cable \emph{in addition} to the global image. Both variants employ hierarchical manipulation to sequentially undo crossings.

\subsubsection{Manipulation}
\label{sec:manipulation}
HULK also uses Node Deletion moves and Reidemeister moves, that are planned directly from bounding boxes around each knot in the configuration and predicted keypoints $\hat{p}_r, \hat{p}_l,\hat{p}_\mathrm{pull}, \hat{p}_\mathrm{pin}$. The keypoints, $p_l$, $p_r$, $p_\mathrm{pull}$, and $p_\mathrm{pin}$, indicate the pixel locations of the left and right cable endpoints $v_l$ and $v_r$, and the pull and pin grasps for the next planned Node Deletion move on the first under-crossing relative to $v_r$, denoted $c$ (Fig.~\ref{fig:untangling_fig}).
 
In a Node Deletion move (Alg.~\ref{alg:hulk_alg}, Ln.~\ref{hulk_undo}) at time $t$, the left arm grasps at $\hat{p}_\mathrm{pull}$ and pulls in the direction of the action vector, $\hat{p}_\mathrm{pull} - \hat{p}_\mathrm{pin}$, and the right arm grasps at  $\hat{p}_\mathrm{pin}$ to pin the cable in place:
\[\mathbf{a_{t,l}} =  ( \hat{p}_{x,\text{pull}}, \hat{p}_{y,\text{pull}},  \hat{p}_{x,\text{pull}} - \hat{p}_{x,\text{pin}}, \hat{p}_{y,\text{pull}} - \hat{p}_{y,\text{pin}}) \hspace{2mm}\vert \hspace{2mm} \mathbf{a_{t,r}} = ( \hat{p}_{x,\text{pin}}, \hat{p}_{y,\text{pin}}, 0, 0)\] 

In a Reidemeister Move (Alg.~\ref{alg:hulk_alg}, Lines \ref{hulk_init_reid} and  \ref{hulk_intermediate_reid}), two consecutive actions pull $\hat{p}_l$ to a predefined point $w_l$ at one end of the workspace, and $\hat{p}_r$ to a predefined point $w_r$ at the opposite end. 
\[\mathbf{a_{t,l}} = ( \hat{p}_{x,l}, \hat{p}_{y,l}, w_{x,l} - \hat{p}_{x,l}, w_{y,l} - \hat{p}_{y,l} ) \hspace{2mm}\vert \hspace{2mm} \mathbf{a_{t,r}} = ( \hat{p}_{x,r}, \hat{p}_{y,r}, w_{x,r} - \hat{p}_{x,r}, w_{y,r} - \hat{p}_{y,r})\] 

The untangling termination condition for BRUCE (Alg.~\ref{alg:untangle_alg}), $|V| > 2$, refers to a configuration free of intersections. For practical implementation, HULK slightly relaxes this constraint to allow intersections that do not form a knot, as taking a  Reidemeister move (Alg.~\ref{alg:hulk_alg}, Ln.~\ref{hulk_intermediate_reid}) will remove any crossings that are not part of a knot. We also want to immediately detect when an endpoint is freed from an under-crossing. However, the bounding box for a knot loosened to this extent is undefined. Thus, we define a condition to approximate when an action pulls the desired endpoint beyond $\hat{p}_\mathrm{pin}$, such that the inner product between the Node Deletion action vector and the vector between $\hat{p}_r$ and $\hat{p}_\mathrm{pin}$ is above a threshold $\lambda$. In experiments, we set $\lambda$ to 0.7 to favor false negatives over positives, as early termination is a greater risk to untangling than late termination. The final termination condition is summarized by Eq.~\ref{eq:termination}, where we additionally impose a hard limit $T = 30$ on the number of actions:
 \begin{equation}
 \label{eq:termination}
      \underbrace{g(I) = \varnothing}_{\text{no knot detected}} \text{ OR } \underbrace{\langle \hat{p}_r - \hat{p}_\mathrm{pin}, \hat{p}_\mathrm{pin} - \hat{p}_\mathrm{pull} \rangle > \lambda}_{\text{endpoint freed from under-crossing}} \text { OR } \underbrace{t > T}_{\text{\# actions exceeded}}.
 \end{equation}

We present a comparison of BRUCE and HULK here, and discuss HULK's perception system below.

\begin{figure}[h!]
\vspace{-12pt}
\MakeRobust{\Call}%
\begin{minipage}[t]{0.45\textwidth}
\begin{algorithm}[H]
\caption{BRUCE}
\label{alg:untangle_alg} 
\begin{algorithmic}[1]
\State \textbf{Input:} Graph $G = (V, E)$ of cable

\State Reidemeister move with $v_r$, $v_l$ 
\While{$|V|>2$}
\State Find $c$ by traversal ($v_r \longrightarrow v_l$)
\State Node Deletion move on $c$
\State Reidemeister move with $v_r, v_l$
\EndWhile
\State \textbf{return} DONE

\end{algorithmic}
\end{algorithm}
\end{minipage}
\begin{minipage}[t]{0.48\textwidth}
\begin{algorithm}[H]
\caption{HULK}
\begin{algorithmic}[1]
\State \textbf{Input:} RGB image of cable

\State Initial Reidemeister move with $\hat{p}_r$, $\hat{p}_l$ \label{hulk_init_reid}
\While{NOT Eq. \ref{eq:termination}} 
\State Directly predict $\hat{p}_\text{pull}$, $\hat{p}_\text{pin}$  \label{hulk_find_knot}
\State Node Deletion move with $\hat{p}_\text{pull}$, $\hat{p}_\text{pin}$\label{hulk_undo}
\State Reidemeister move with $\hat{p}_r$, $\hat{p}_l$ \label{hulk_intermediate_reid}
\EndWhile
\State \textbf{return} DONE
\end{algorithmic}
\label{alg:hulk_alg} 
\end{algorithm}
\end{minipage}
\vspace{-6pt}
\end{figure}

\subsubsection{Perception}
HULK employs perception-based learning to robustly infer the task-relevant components of the cable as defined in Section~\ref{sec:manipulation}. 
HULK-G infers all keypoints $p_l$, $p_r$, $p_\mathrm{pull}$, and $p_\mathrm{pin}$ from a full-resolution RGB image of a cable. HULK-L also infers $p_l$, $p_r$ globally, but additionally uses local information to estimate $p_\mathrm{pull}$, and $p_\mathrm{pin}$ directly from a bounding box crop of the right-most knot.

\textbf{Knot Detection}: HULK uses a bounding box knot-detector to instantiate a termination condition for the task. HULK-L additionally uses this module to predict $\hat{p}_\mathrm{pull}$, and $\hat{p}_\mathrm{pin}$ local to knot crops instead of the global image as in HULK-G. We learn a function $g: \mathbb{R}^{640 \times 480 \times 3} \mapsto \mathbb{R}^{4 \times N}$ that maps a full-resolution RGB image to bounding boxes $(x_{\text{min}}, y_{\text{min}}, x_{\text{max}}, y_{\text{max}})_i|_{i=1,\hdots N}$ for $N$ knots contained in the image using a Mask-RCNN framework as in \citet{he2017mask} (Fig.~\ref{fig:perception_fig}). 

\textbf{Keypoint Regression}: HULK-L and HULK-G implement separate keypoint regression modules. The distinction is that HULK-L attempts to take advantage of local information by explicitly using predicted knot bounding boxes. HULK-G learns a function $f_\mathrm{all}$ that globally regresses all 4 keypoints as 2D Gaussian heatmaps (Fig. \ref{fig:perception_fig}). Provided a global image of the cable, $I \in \mathbb{R}^{640 \times 480 \times 3}$, HULK-G computes $\hat{p}_l, \hat{p}_r, \hat{p}_\mathrm{pull}, \hat{p}_\mathrm{pin} = f_\mathrm{all}[I]$.

HULK-L, based on on \citet{papandreou2017towards} consists of one network $f_\mathrm{reid}$ which infers $p_r,p_l$  from the global image and $f_\mathrm{node}$ which infers $p_\mathrm{pull}, p_\mathrm{pin}$ from a local predicted bounding box knot crop of size (80,60) (Fig.~\ref{fig:perception_fig}~D). HULK-L computes $\hat{p}_l, \hat{p}_r = f_\mathrm{reid}[I]$, and predicts the rightmost bounding box $\hat{b}$ from g(I) (Fig.~\ref{fig:perception_fig}), consistent with the assumption that we always untangle by finding the first under-crossing relative to the right endpoint. Finally, HULK-L finds $\hat{p}_\mathrm{pull}, \hat{p}_\mathrm{pin} = f_\mathrm{node}[I[\hat{b'}]]$ locally within $\hat{b'}$, where $\hat{b'}$ is $\hat{b}$ resized to aspect ratio (80,60). Additional details about both HULK-G and HULK-L are provided in Appendix \ref{sec:learning_details}. 

\begin{figure}[!thbp]
    \centering
    \vspace{-0.2cm}
    \includegraphics[width=0.98\linewidth]{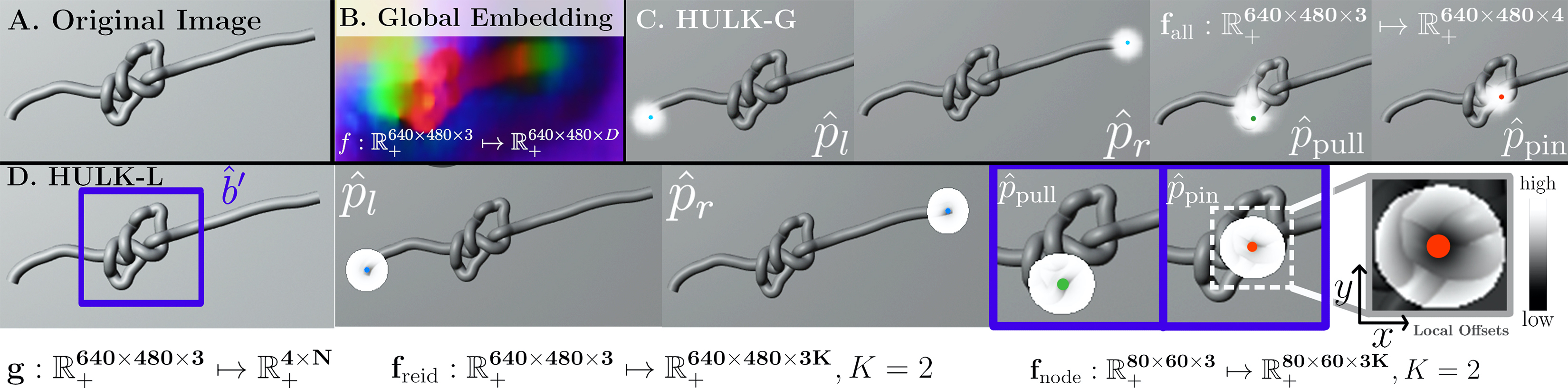}
    \caption{\textbf{Perception Overview}: We compare two implementations of HULK's perception (C, D) against full state estimation approaches on an RGB image of a simulated cable (A). We train a dense object descriptor network~\cite{sundaresan2020learning, florence2018dense, schmidt2016self} that learns a dense mapping to a unit-normalized 3D descriptor space visualized in RGB (B). The mapping fails to discriminate between overlapping segments, shown by the lack of color distinction around the knot, preventing reliable inference of task-relevant keypoints. HULK-G (C) outputs 4 Gaussians centered around $p_l, p_r, p_\mathrm{pull}, p_\mathrm{pin}$ (left to right). HULK-L (D) learns $f_\mathrm{reid}, f_\mathrm{node}$ from global images and knot-crops respectively, which each produce 3 heatmaps per keypoint: binary classification of pixels within a radius $R$ of the keypoint, pixelwise $x$ offsets, and pixelwise $y$ offsets. The final predicted keypoint is a positively classified pixel with minimum predicted offset to the desired keypoint (Local Offset subfigure).}
    \label{fig:perception_fig}
    \vspace{-0.3cm}
\end{figure}


\section{Experiments}
\label{sec:experiments}
We evaluate HULK in simulation and physical experiments of untangling cables.
We test all policies on 3 different dense starting knot configurations (single overhand, single figure-eight, overhand in series with figure-eight). In simulation, we compare HULK-L, HULK-G, and 3 baseline policies and consider 3 different textures (cylindrical capsules, smooth, braid) as shown in Fig.~\ref{fig:simulator_fig}. In physical experiments, we compare HULK-G with 1 baseline policy.
\vspace{-0.15in}
\paragraph{Overview of Policies:}
All policies take as input an image and compute bounding boxes from RGB image inputs and use Eq.~\ref{eq:termination} as a termination condition. Each policy varies in its estimation of $p_l, p_r,  p_{\text{pull}}, p_{\text{pin}}$ which are used to implement Reidemeister and Node Deletion moves.
\begin{itemize*}[label=,afterlabel=]
    \item The \textbf{Oracle} baseline implements BRUCE with \emph{full access} to the ground-truth 3D cable state in simulation.
    \item The \textbf{Depth} baseline requires RGB-D perception and approximates $p_{\text{pin}}$ as the highest depth pixel within $\hat{b}$ and $p_{\text{pull}}$ is 15 pixels to the left of $p_{\text{pin}}$ since we untangle relative to the right endpoint.
    \item  The \textbf{Random} baseline requires both RGB and the cable binary segmentation mask, and takes arbitrary actions sampled on the cable segmentation mask within $\hat{b}$.
    \item \textbf{HULK-G} uses learned keypoint estimation (Fig.~\ref{fig:perception_fig}) to detect endpoints and pull/pin locations from full-workspace RGB images.
    \item \textbf{HULK-L} combines bounding box knot detection, endpoint and local pull/pin detection (Fig.~\ref{fig:perception_fig}) from RGB images.
\end{itemize*}

\subsection{Simulation Experiments}
In HULK-L, for each of the three textures (Fig.~\ref{fig:simulator_fig}), we learn a separate set of models
using the procedure described in Section~\ref{sec:implementation}. In HULK-G, we learn a separate bounding box knot-detection model and global keypoint regression model for endpoints and pull/pin locations for each texture. We train each network on 3500 rendered synthetic images of the appropriately textured cable in randomized initial knotted configurations. We run 21 trials in simulation for all combinations of manipulation policy, starting knot configuration, and texture. 
\begin{figure}[!htbp]
    \vspace{-0.1in}
    \centering
    \includegraphics[width=1.0\linewidth]{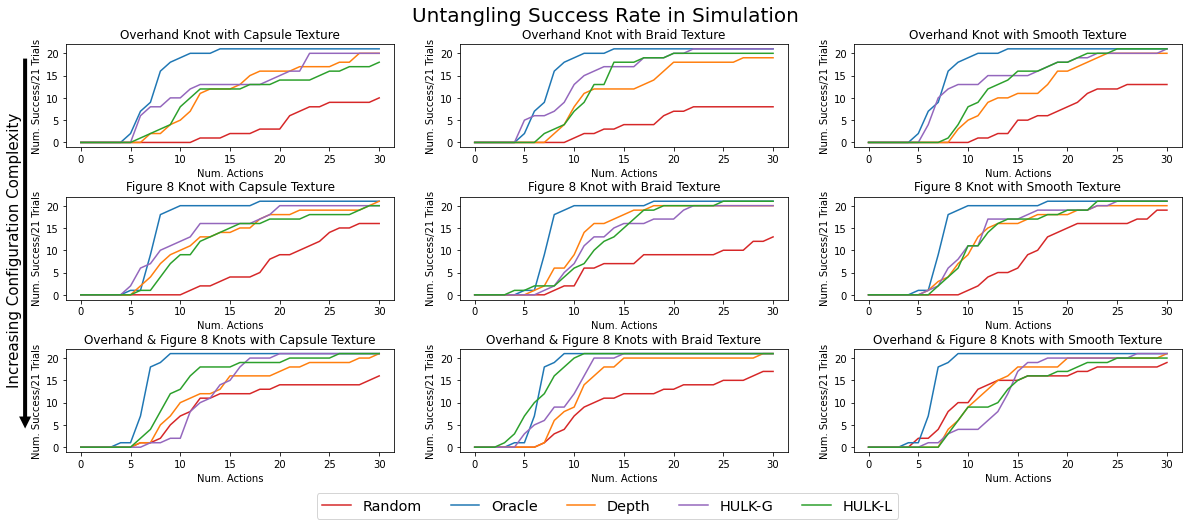}
    \caption{\textbf{Simulation Results}: Untangling success rate and efficiency over 21 trials are plotted across each texture, initial configuration, and policy. We find that the bounding box knot detector performs best with the smooth texture relative to the braid and capsule textures because the knot geometry is easiest to perceive in the smoother texture. Both the global and local keypoint detectors perform worse on the smooth texture due to the lack of features along the cable to disambiguate crossings.
    }
    \label{fig:experiments_fig}\vspace{-0.2in}
\end{figure}

Failure modes of HULK arise in separate components of the perception method enabling modular improvement approaches. In the bounding box detection module, false negatives can lead to premature termination. In HULK-L, since the pull/pin keypoint regression module is conditioned on an accurately predicted bounding box, erroneously predicted bounding box dimensions or false positives can result in a crop that either does not include the relevant under-crossing or captures too large of a crop. Failures in the keypoint module typically occur in predicting a pull/pin pair rather than endpoints, when the network identifies keypoints local to a crossing that is not the next immediate under-crossing. Other failure modes include premature or late termination caused by a mispredicted action that erroneously triggers the termination condition.

We evaluate HULK and baselines by measuring the success rate of untangling as a function of the number of actions taken (Fig.~\ref{fig:experiments_fig}), and benchmarking all engineered policies relative to Oracle. HULK-L and HULK-G outperform Random across all textures and initial knot configurations with empirically higher efficiency and higher success, as shown in the steeper convergence of both HULK implementations compared to Random in all plots in Fig.~\ref{fig:experiments_fig}. HULK-G and HULK-L match the performance of the Depth baseline in several trials, even though both HULK-L and HULK-G lack depth information, and the Depth baseline has access to the ground-truth endpoints which is critical to preventing premature termination. The Depth baseline also naively pulls directly left in all cases, which is not camera-pose agnostic, and depth provides a strong noise-free signal in simulation. HULK-L outperforms HULK-G in the 2 out of 3 multiple knot experiments as its hierarchical perception encourages the prediction of $\hat{p}_\text{pull}$ and $\hat{p}_\text{pin}$ around the correct crossing.

\subsection{Physical Experiments}
As shown in Fig.~\ref{fig:simulator_fig}, we evaluate HULK using the two 7-DOF arms of a da Vinci Research Kit (dVRK) surgical robot~\cite{chen2017software} to untie an elastic hair tie.
Empirically, the hair tie is more pliable than other cables, allowing for easier pinning and pulling, and is proportionate in scale to the dVRK end-effectors. The hair tie is pre-cut, 5~mm in diameter, 15~cm in length, and smooth in texture. We use a Zivid OnePlus RGBD camera (Figure \ref{fig:untangling_fig}) to collect 1920x1200 RGB overhead images, as the depth channel is not used in perception inference.

We train the keypoint and bounding box networks of HULK-G on 350 hand-labelled real hair tie images, augmented 20x according to Section~\ref{sec:implementation}. Details for the experimental setup, including grasp-planning, collision handling, and image-based inference are in Appendix~\ref{sec:experimental_details}. As a baseline, we implement the Random policy described above, but use color-based thresholding to approximate cable/knot segmentation masks as ground truth segmentations are unavailable in a real setting. The Random policy also approximates the cable endpoints as the extreme points on the cable mask for Reidemeister moves. We run 10 trials of each method on three starting configurations and with two tiers of difficulty: dense starting configurations and looser starting configurations, as shown in Fig.~\ref{fig:phys_exp_setup}. We define a success to be an ending configuration with no knots and at most one crossing, given the tendency of the elastic hair tie to spring back into single crossing configurations at rest. The results reported in Table~\ref{table:phys_exp_results} suggest the effectiveness of HULK-G over a Random baseline in untangling knots in a hair tie, particularly for dense knots. Most manipulation errors are due to grasp-planning and cable-driven robot dynamics, detailed in Appendix~\ref{sec:experimental_details}, leading to gripper slippage or near misses during grasping. We observe similar perception failure modes as in simulation experiments: false negatives in the bounding box detector, leading to early stopping, and mispredicted actions from poor keypoint predictions, which are less costly.

\begin{table}[!htbp]
\centering
\vspace{-0.2cm}
\resizebox{\columnwidth}{!}{
 \begin{tabular}{||c | c || c | c | c || c | c | c ||} 
 \hline
 \multicolumn{2}{||c||}{}&
 \multicolumn{3}{c||}{Random}& \multicolumn{3}{c||}{HULK-G}\\ 
 \hline
Density & Knot & Success Rate & Actions / Success & Failure Modes & Success Rate & Actions / Success & Failure Modes \\ 
\hline
\hline
Loose & O & 3/10 & 6.67 & A (1), B (0), C (4), D (2) & \textbf{7/10} & \textbf{4.71} & A (1), B (\textbf{0}), C (1), D (1) \\
\hline
Loose & F & 6/10 & 5.67 & A (0), B (1), C (3), D (0) & \textbf{7/10} & \textbf{3.14} & A (1), B (1), C (1), D (\textbf{0}) \\
\hline
\hline
Dense & O & 0/10 & N/A & A (4), B (0), C (4), D (2) & \textbf{5/10} & \textbf{6.60} & A (\textbf{0}), B (\textbf{0}), C (3), D (2) \\
\hline
Dense & F & 2/10 & 12.5 & A (1), B (1), C (5), D (1) & \textbf{7/10} & \textbf{7.57} & A (\textbf{0}), B (\textbf{0}), C (3), D (\textbf{0}) \\
\hline
Dense & O + F & 0/10 & N/A & A (4), B (2), C (4), D (0) & \textbf{5/10} & \textbf{17.8} & A (\textbf{0}), B (\textbf{0}), C (4), D (1) \\
\hline
Dense & O + O & 0/10 & N/A & A (1), B (2), C (7), D (0) & \textbf{6/10} & \textbf{14.0} & A (\textbf{0}), B (\textbf{0}), C (3), D (1) \\
\hline
\end{tabular}}
\\

\caption{\textbf{Physical Results:} We report the success rate and efficiency for untangling an elastic hair tie cable containing Overhand (O) and Figure-8 (F) knots on the dVRK.
Keypoint inference on an Nvidia GeForce RTX 2080 takes 114~ms, bounding box inference takes 315~ms, Node Deletion motions take 10~s to execute, and Reidemeister motions take 3.5~s per gripper. We observe HULK-G can generalize to Overhand + Overhand (O + O), which is not seen during training. We categorize the occurrence of 4 failure modes: (A) early stopping due to bounding box false negatives, (B) the cable leaving the reachable workspace due to consecutive mispredicted actions, (C) the cable getting caught in the end-effector jaws due to high hair tie friction and density, and (D) exceeding the maximum number of allowed actions (15 per knot). In double-knot experiments, we observe fewer mode A failures as the double-knot reduces the likelihood of false negatives, but the longer task horizon makes failure mode C more prevalent.}
\label{table:phys_exp_results}
\vspace{-0.3cm}
 \end{table}
 

\section{Conclusion}
\label{sec:conclusion}
We present BRUCE, a geometric algorithm for untying loose knots, and HULK, a learned perception-based algorithm to untie dense overhand and figure-eight knots from visual input. HULK keypoint outputs are used to plan straightening and loosening actions. We evaluate the effectiveness of HULK against 3 baselines in simulation and 1 in physical trials. Simulation experiments suggest that HULK trained with BRUCE as an algorithmic supervisor can more successfully and efficiently untangle cables consisting of varied textures and initial configurations compared to analytical approaches. Physical trials with 350 hand-labelled training examples and robot untying trials on 60 knotted cables averaging 3-18 actions per trial suggest that HULK can successfully untie dense knots in well over 50\% of trials, with most failures due to mechanical effects. In future work we will extend HULK to more complex knot configurations, cables that vary in physical properties and sizes, and employ active sensing to handle uncertain states encountered during manipulation. We hope to make the acronym HOUDINI work for the next version of the algorithm — HULK out.   


\clearpage
\footnotesize
\acknowledgments{
This research was performed at the AUTOLAB at UC Berkeley in affiliation with the Berkeley AI Research (BAIR) Lab,  the Real-Time Intelligent Secure Execution (RISE) Lab
and the CITRIS "People and Robots" (CPAR) Initiative. Any opinions, findings, and conclusions or recommendations expressed in this material are those of the author(s) and do not necessarily reflect the views of the sponsors. The authors were supported in part by donations from Toyota Research Institute, Google, and by equipment grants from Intuitive Surgical. The da Vinci Research Kit is supported by the National Science Foundation, via the National Robotics Initiative (NRI), as part of the collaborative research project "Software Framework for Research in
Semi-Autonomous Teleoperation" between The Johns Hopkins University (IIS 1637789), Worcester Polytechnic Institute (IIS 1637759), and the University of Washington (IIS 1637444). Ashwin Balakrishna is supported by an NSF GRFP. We thank our colleagues who provided helpful feedback, code, and suggestions, especially Kate Sanders, Wisdom Agboh, and Daniel Seita.}

\begin{small}
\bibliography{corl}  
\end{small}
\normalsize
\newpage
\appendix
\begin{LARGE}
\begin{center}
\textbf{Learning Robot Policies for Untangling Dense Knots in Linear Deformable Structures \\ Supplementary Material}
\end{center}
\end{LARGE}
\maketitle

The supplementary material is structured as follows: Appendix \ref{sec:sim_details} contains additional details on the cable simulator. Appendix \ref{sec:learning_details} specifies details and hyperparameters for all learning-based methods. Appendix \ref{sec:experimental_details} contains additional details on physical and simulation experiments. 

\section{Simulator Details}
\label{sec:sim_details}
While out-of-the-box models exist for both cloth and soft-body objects, the implementations in Blender and other widely used simulators (e.g. Mujoco ~\cite{todorov2012mujoco} and PyBullet ~\cite{coumans2016pybullet}) do not serve our purposes for simulating realistic physical behavior and visual appearance of a densely knotted cable. Furthermore, they do not easily allow us to collect self-supervised labels for keypoints and bounding boxes on cable features such as knots and crossings. Thus, we implement a custom simulator with a cable modeled as a set of rigid-body cylinders linked by springs. The cable is manipulated on a surface implemented as a rigid-body plane with friction in Blender 2.8. The hyperparameters of the simulation are experimentally chosen to meet the tradeoffs between suitable simulation realism and rendering times and are presented in Table~\ref{table:simulator_params}. 

\begin{table}[H]
\centering
\begin{tabular}{||m{0.25\linewidth}|| m{0.5\linewidth} || m{0.1\linewidth} ||}
 \hline
 \textbf{Simulation Parameter} & \textbf{Explanation} & \textbf{Value} \\
  \hline
   \hline
 Number of cylinders & \# of segments comprising cable & 50 \\
  \hline
Cylinder radii & radial thickness of cable & 0.25m \\
 \hline
 Cylinder lengths & length of individual segments & 1m \\
 \hline
 Cylinder masses & masses of individual segments & 0.05kg \\
  \hline
  Cylinder friction & friction of individual segments & 1.0N \\
 \hline
   Spring linear damping & proportion of linear velocity that is lost over time & 0.55 \\
 \hline
    Spring angular damping & proportion of angular velocity that is lost over time  & 0.55 \\
 \hline
     Workspace friction & friction on manipulation surface  & 1.0N \\
 \hline
\end{tabular}
\caption{\textbf{Blender 2.8 cable simulation Hyperparameters}}
\label{table:simulator_params}
\vspace{-0.5cm}
 \end{table}
 
We use an external textured mesh (currently either braided or smooth) that deforms according the movement of rigid-body cylinders. The simulator can be extended to support a new texture provided a mesh.
 
 \subsection{Manipulation}
 \label{sec:sim_manip}
We employ Blender's keyframe system to manipulate cable's for data generation and untangling trials. Given a keyframe at the current object pose (location and orientation) and a keyframe at the desired object pose at a time in the future, Blender will interpolate the pose of the object for all frames between the two keyframes. We use this functionality to expose an API for manipulating the cable in simulation by keyframing its individual cylindrical segments. A unilateral pick-place action (with a fixed vertical offset) is parameterized as follows (Section \ref{sec:ps}): $\mathbf{a_{t}} = (x_{t}, y_{t}, \Delta x_{t}, \Delta y_{t})$. To implement this primitive for an arbitrary $(x_{t}, y_{t}, \Delta x_{t}, \Delta y_{t})$, we perform nearest neighbors search to locate the cylindrical segment nearest $(x_{t}, y_{t})$, and keyframe its updated location  $(x_t + \Delta x_{t}, y_t + \Delta y_{t})$ at time $t+T$, where $T$ is a fixed action duration. To implement a \textbf{Reidemeister move}, we perform unilateral pick-place actions on the left and right endpoint cylinders in sequence, pulling them to opposing, pre-defined locations in the workspace and releasing. For \textbf{Node Deletion} moves, we simultaneously take a pick-place action on a pull cylinder while keyframing the pin cylinder in place. 

We also execute bilateral knot-tying trajectories using keyframing to produce overhand and figure-eight knots. For each knot type, we plan a trajectory for the left and right endpoints by visual inspection to yield the desired knot, as shown for an overhand knot in Figure \ref{fig:overhand_trajectory}. Since the trajectories are deterministic, we implement three randomization techniques to yield perturbed initial configurations both at train and test time: (1) adding random noise to the initial trajectory, (2) lifting one endpoint, forcing gravity to randomly slide a knot lengthwise down the cable, and (3) taking a random pick-place action on the cable.

\begin{figure}[H]
    \centering
    \includegraphics[width=0.8\linewidth]{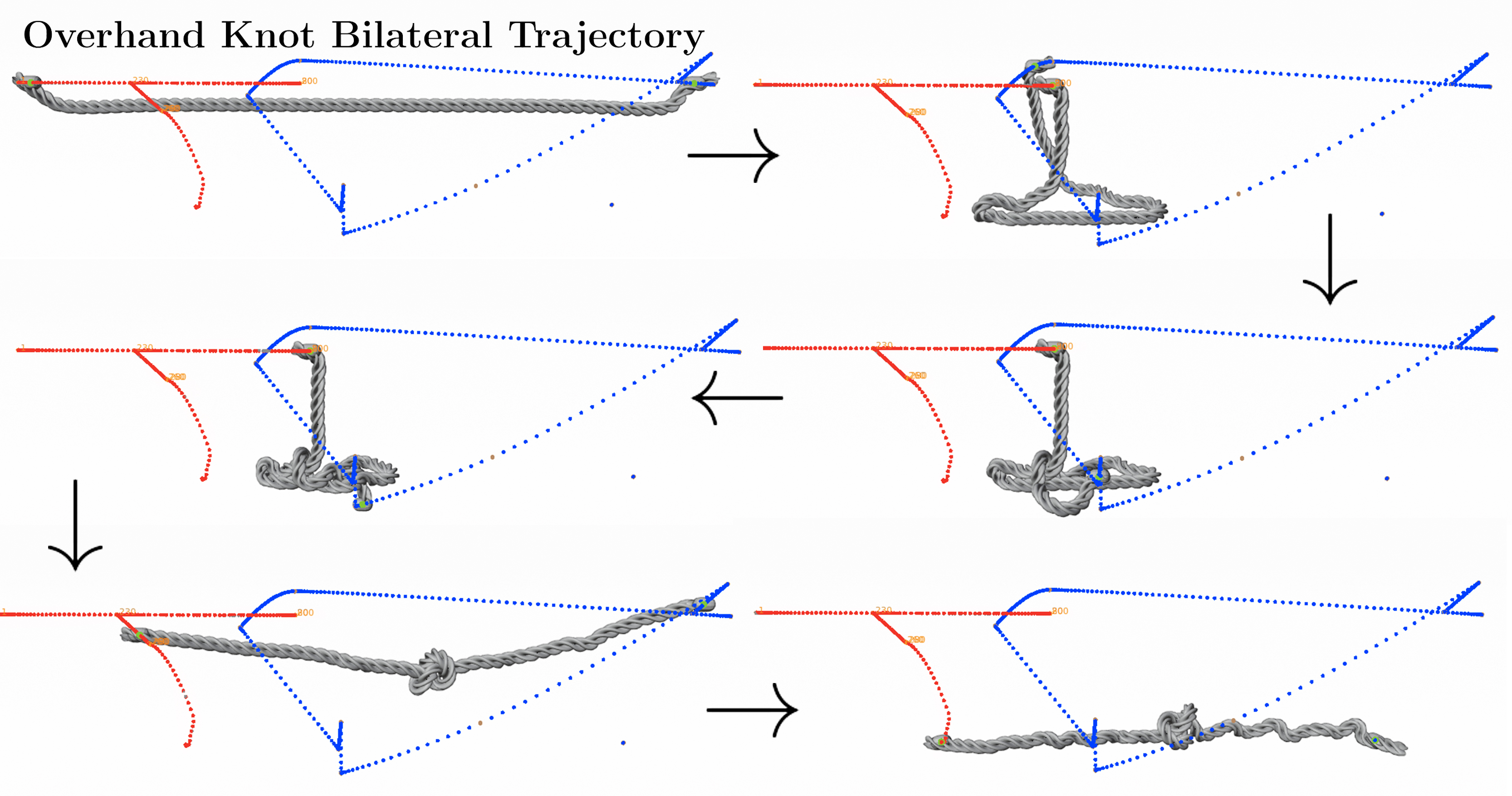}
    \caption{\textbf{Overhand Knot-Tying Trajectory}: We tie an overhand knot by simultaneously keyframing the poses of the left endpoint (red trajectory) and right endpoint (trajectory). The right endpoint is first wrapped around the left endpoint (row 1) and threaded through the resulting loop (row 2). Finally, the endpoints are pulled taught and then released (row 3).
    }
    \label{fig:overhand_trajectory}
\end{figure}

\subsection{Perception}
In this section, we describe the necessary perception components of the simulator used to train HULK. This includes the implementation of BRUCE as an algorithmic supervisor that is used in dataset generation with self-supervised annotation collection.

\subsubsection{BRUCE}

\subsection{Completeness of BRUCE}
In this section, we will provide conditions under which BRUCE is guaranteed to untangle a cable in a semi-planar configuration.

BRUCE takes an input graph $G_0$ and performs a series of $k$ moves resulting in graph $G_k$. Each move $i$ produces intermediate graph $G_i$. 
\\ \textbf{Input}: Graph $G_0 = (E_0, V_0)$ with $N_0$ nodes (crossings).
\\ \textbf{Output}: Graph $G_k = (E_k, V_k)$ with $N_k$ nodes (crossings), where $N_k = 0$. 

We assume the following:
\begin{assumption}\label{assumption:node-reduction}
Each node deletion move reduces $N_i$ by 1.
\end{assumption}
\begin{assumption}\label{assumption:reid-reduction}
Each Reidemeister move reduces $N_i$ by $\geq 0$.
\end{assumption}
\begin{assumption}\label{assumption:consecutive-reid}
It is not possible to take two consecutive Reidemeister moves. 
\end{assumption}
\begin{assumption}\label{assumption:termination}
BRUCE can detect and terminate when $N_k = 0$.
\end{assumption}
The first and second assumptions respectively mean that removing an under-crossing or pulling the ends apart does not result in a more knotted configuration. The third assumption is based on Algorithm \ref{alg:untangle_alg}, in which a Reidemeister move is always either the first move taken or is preceded by a Node Deletion move. 

\begin{lemma}
Let the starting configuration of a cable have $N_0\geq 0$ crossings. If Assumptions~\ref{assumption:node-reduction}-\ref{assumption:termination} hold, then BRUCE is guaranteed to successfully terminate in a state with $N_k=0$ (no crossings) in at most $2N_0$ moves.
\end{lemma}

\begin{proof}
From the assumptions, it follows that $N_i \leq N_{i-1}$ and $N_i < N_{i-2}$. As $N_i$ cannot increase, and must decrease by 1 at least every other move, BRUCE must terminate with $N_k = 0$ after at most $k = 2N_0$ moves, completing the proof.
\end{proof}

We note that proving similar properties for the practical algorithm HULK is significantly more challenging, as this requires reasoning about the accuracy of the perception systems used to implement the moves and termination condition. Additionally, we note that Reidemeister moves can also reduce the number of nodes in the graph representation, which suggests that a graph-specific upper bound on the runtime might be lower. For instance, a Reidemeister move on a coiled rope will immediately untangle it, as there are no knots in its configuration.



We implement BRUCE in simulation as an algorithmic supervisor for training HULK. 
We assume the full state of the cable is given by $N=50$ cylinders of the cable, and their ordered 3D locations $\{c_i, (x_i, y_i, z_i)\}\vert_{i \in 1 \hdots N}$. This ground truth information is used to implement a heuristic approximation of BRUCE (Algorithm \ref{alg:bruce_sim_alg}) using full knowledge of the cable state and an algorithm to iteratively detect under-crossings (Algorithm \ref{alg:uc_detection}). We locate the first under-crossing by finding a cylinder pair that (1) has a sufficient depth difference, and (2) is separated lengthwise along the planar cable, representing two different pieces of the cable. Given this procedure, we implement a best-approximation of BRUCE:

\begin{algorithm}[H]
\caption{Heuristic BRUCE}
\label{alg:bruce_sim_alg} 
\begin{algorithmic}[1]
\State \textbf{Require:} $\{c_i, (x_i, y_i, z_i)\}\vert_{i \in 1 \hdots N}$
\State Reidemeister move with $c_1$, $c_{N}$ 
\While{\texttt{True}}
\State \texttt{(pull\_idx, pin\_idx)} $\gets$ \textsc{FIND\_UNDER\_CROSSING}($(x_i, y_i, z_i)\}\vert_{i \in 1 \hdots N}$)
\If{\texttt{(pull\_idx, pin\_idx)} $\neq$ \texttt{(None,None)}}
\State $c_\mathrm{pull} = c_{\texttt{pull\_idx}}$
\State $c_\mathrm{pin} = c_{\texttt{pin\_idx}}$
\State Node Deletion move with $c_\mathrm{pull}, c_\mathrm{pin}$
\State Reidemeister move with $c_1, c_{N}$
\Else
\State \texttt{break} \Comment{No under-crossing found, exit}
\EndIf
\EndWhile
\State \textbf{return} DONE
\end{algorithmic}
\end{algorithm}

\begin{algorithm}[H]
\caption{Under-crossing Detection}
\label{alg:uc_detection}
\begin{algorithmic}[1]
\State \textbf{Require:} $\{(x_i, y_i, z_i)\}\vert_{i \in 1 \hdots N}$
\State \LineComment{Initialize experimentally-tuned hyperparameters} 
\State \texttt{DEPTH\_THRESH} $\gets 0.4$
\State \texttt{IDX\_THRESH} $\gets 3$
\State \texttt{PULL\_OFFSET} $\gets 3$
\State \LineComment{\# \# \#}
\For{$i \gets 1$ to $N$} \Comment{For each cylinder, find NN in $x-y$ plane (to detect intersections)}
    \State \texttt{$(x_\mathrm{j}, y_\mathrm{j}) \gets$} 2nd nearest neighbor for $(x_i, y_i)$ in (x,y) projection \Comment{1st match is always $i$ (identical)}
    \State \texttt{depth\_diff} $\gets z_\mathrm{j} - z_i$
    \State \texttt{idx\_diff} $\gets j - i$
    \If{\texttt{depth\_diff} $>$ \texttt{DEPTH\_THRESH} \texttt{AND} \texttt{idx\_diff} $>$ \texttt{IDX\_THRESH}}
    \State \LineComment{Found $c_i, c_j$ (under, over) w/ sufficient depth diff and geodesic separation}
    \State \texttt{pull\_idx} $\gets i + $ \texttt{PULL\_OFFSET} \Comment{Select unoccluded pull point just beyond crossing}
    \State \texttt{pin\_idx} $\gets j$ \Comment{Pin is found match}
    \State \textbf{return} \texttt{(pull\_idx, pin\_idx)} \Comment{Cylinder pair for first under-crossing}
    \EndIf 
\EndFor
\State \textbf{return} \texttt{(None, None)} \Comment{No under-crossing detected}
\end{algorithmic}
\end{algorithm}

\subsubsection{Dataset Generation}
\label{sec:dset_gen}
We use the implementation of Heuristic BRUCE (Algorithm \ref{alg:bruce_sim_alg}) to collect self-supervised annotations for training both the bounding-box knot-detection module and the keypoint regression modules of HULK. The procedure for generating training data is to tie a randomized knot trajectory as described in Section \ref{sec:sim_manip} and execute Heuristic BRUCE to untangle the knot while recording rendered RGB images and their corresponding keypoint annotations and bounding-box annotations. 

The keypoint annotations for $p_l, p_r, p_\mathrm{pull}, p_\mathrm{pin}$ are straightforward to find by projecting the 3D locations of the cylinders $c_1, c_N, c_\mathrm{pull}, c_\mathrm{pin}$ to 2D pixels with the synthetic overhead camera. Collecting annotations of knot bounding boxes is nontrivial since it requires defining a knot in terms of the cylinders. We assume that a knot consists of cylinders in the vicinity of a found under-crossing. We run Algorithm \ref{alg:uc_detection} on $(x_i, y_i, z_i)\}\vert_{i \in 1 \hdots N}$ to first find the pull and pin cylindrical indices \texttt{(pull\_idx, pin\_idx)} for a given cable configuration. We specify the knot with an experimentally tuned buffer of 4 cylinders before and after the pull and pin cylinders. Thus, we consider a knot to be the set of cylinders $\{c_{\texttt{pull\_idx}-4}, \hdots, c_{\texttt{pull\_idx}-1}, c_{\texttt{pull\_idx}}, c_{\texttt{pin\_idx}}, c_{\texttt{pin\_idx}+1}, \hdots, c_{\texttt{pin\_idx}+4} \}$. Finally, we compute the bounding box $x_\mathrm{min}, y_\mathrm{min}, x_\mathrm{max}, y_\mathrm{max}$ in pixel-space, taken over all 2D pixel projections of the knot-containing cylinders. For the multi-knot case, we repeat this annotation process with input $(x_i, y_i, z_i)\}\vert_{i \in \texttt{pin\_idx} \hdots N}$ to Algorithm 3 in order to find the next knot.
 
\section{Details of Learning-Based Methods}
\label{sec:learning_details}
In this section, we outline the implementation and hyperparameters of all the learned components of HULK.
\subsection{Bounding-Box Knot Detection}
We train a bounding box knot-detection module used both for the termination condition of generic HULK (Equation \ref{eq:termination}) and for the local keypoint-regression module of HULK-L (Section \ref{sec:implementation}). The network is based on Matterport's implementation of Mask-RCNN from ~\citet{he2017mask} in Tensorflow and Keras open-sourced on Github \footnote{\url{https://github.com/matterport/Mask_RCNN}}. The implementation uses a ResNet-101 and FPN backbone with the SGD optimizer configured with a learning rate of 0.001, momentum of 0.9, and weight decay of 0.0001. The network is trained on images of $640 \times 480$ RGB images of cables with bounding box annotations recorded according to the procedure described in Section \ref{sec:dset_gen}. Training with 3,500 images for 10 epochs takes approximately 30 minutes on a Nvidia GeForce GTX 1070 GPU. We experimented with the classification threshold at the output and settled on $0.94$, which seemed to suppress false negatives (which are particularly bad for early termination during untangling) without under-predicting bounding boxes. For the network trained on synthetic data, we report a mean average precision (mAP) of 0.925 on the training dataset and 0.916 for a test set of 686 images.

\subsection{HULK-G Keypoint Regression}
We discuss the implementation of HULK-G, which learns to predict the keypoints $p_l, p_r, p_\mathrm{pull}$ and $p_\mathrm{pin}$ from a global RGB image of a cable. HULK-G learns a mapping $f_{\text{all}}: \mathbb{R}^{640 \times 480 \times 3} \mapsto \mathbb{R}^{640 \times 480 \times 4}$ where each of the four channel outputs are a 2D Gaussian centered at $p_l, p_r, p_\mathrm{pull}$ and $p_\mathrm{pin}$ (Fig. \ref{fig:perception_fig}). The network for HULK-G uses a stride-8 ResNet-34 backbone and bilinearly upsamples the output to the original $640 \times 480$ input image size, extended across $4$ channels for the $4$ keypoints. Originally, we experimented with a ResNet-18 backbone which exhibited high variance for the pull and pin keypoints specifically on training images, so we tried a deeper architecture with 34 layers which exhibited better qualitative performance on both training and test images. In practice, we found that a large Gaussian $\sigma$ led to overlapping regions in the pull and pin distributions, preventing fine-grained localization. In contrast, the network became biased to underpredicting all keypoints with a low $\sigma$ since there were very few regions of high probability density in the ground-truth distributions. We train HULK-G with RGB images (per each texture considered) with a Gaussian $\sigma$ of 8 pixels for the ground-truth distributions. The network is trained with a batch size of 4 (for the sake of efficiency) using the Adam optimizer with a learning rate of 0.0001. We found that early stopping around 4 epochs, a train time of ~45 minutes (for 3,500 train images) on an Nvidia Tesla V-100 GPU, prevented overfitting.
\subsection{HULK-L Keypoint Regression}
Since we are using bounding-boxes for HULK's termination condition regardless, a natural extension was to predict the pull and pin keypoints within the bounding box of a knot. HULK-L also uses a different learning strategy that prioritizes local information over HULK-G, motivated by the variance observed in HULK-G for predicting pull and pin points. HULK-L, based on ~\citet{papandreou2017towards}, learns two mappings (1) $f_\mathrm{reid}:  \mathbb{R}^{640 \times 480 \times 3} \mapsto \mathbb{R}^{640 \times 480 \times (3 \times 2)}$ for inferring $p_r,p_l$  from the global image and (2) $f_\mathrm{node}: \mathbb{R}^{80 \times 60 \times 3} \mapsto \mathbb{R}^{80 \times 60 \times (3 \times 2)}$ for inferring $p_\mathrm{pull}, p_\mathrm{pin}$ from a local knot crop of size (80,60) (Fig.~\ref{fig:perception_fig}~D). The output for each network contains 3 channels extended across 2 keypoints; the first channel performs pixelwise binary classification, classifying pixels as 1 if they lie in a disc of radius $R$ containing the desired point, and the remaining 2 channels regress the $x,y$ offsets for each pixel relative to the point of interest. This learning approach is intended to finely localize keypoints by first narrowing the best match candidates to a disc around the desired point and from there find the pixel with the least predicted $x,y$ offsets to the desired point. We use a ResNet-18 backbone for the network with binary cross-entropy loss for the classification heatmap and smooth $L^1$ loss for the offset regression outputs, with a learning rate of 0.0001 trained for 20 epochs (~45 minutes on Nvidia Tesla V-100 GPU). We found experimentally that this combination of hyperparameters yielded the best generalization to unseen test images.

\textbf{HULK-G vs. HULK-L}: In simulated experiments (\ref{sec:experiments}), we did not observe a noticeable advantage to using HULK-L or HULK-G over the other; while HULK-L is able to exploit local information more effectively than HULK-G, HULK-L is also more susceptible to compounding errors over time, since it is conditioned on accurate knot bounding box detection. This becomes increasingly harder as untangling progresses. However, the ability of both HULK-L and HULK-G to achieve untangling suggests that learning task-specific keypoints, globally or locally, is an effective strategy when coupled with geometrically designed manipulation policies.

\section{Experimental Details}
\label{sec:experimental_details}
\subsection{Quantifying Density} 
We first define $E'$ to be every point in the subspace encompassing the physical length and radial thickness of the cable. 
To quantify the density of knotted configurations, we define a looseness measure $L(E^*) \in \mathbb{R}$. $L(E^*)$ quantifies the minimum amount of empty space between adjacent cable segments. We consider point-pairs $(p, p')$ sampled from a subspace $E^*(u, v) \subset \mathbb{R}^3$ of points within 2 cable segments joining 2 vertices $u$ and $v$. $L(E^*)$ is proportional to the maximum Euclidean distance taken over all point pairs sampled from $E^*(u,v)$, minimized over all vertex pairs $(u,v)$ that share 2 edges. For convenience, we denote the points within the 2 edges between $(u,v)$ as $\{e^*_{+}, e^*_{-}\}$ where $e^*_{+}$ corresponds to points within the cable segment corresponding to the edge with annotation $X(u, e_{+}) = +1$ and the opposite notation applies for $e^*_{-}$.
\begin{align}
\label{eq:looseness}
L(E^*) & =  \min_{\{e^*_{+},e^*_{-}\} \subset E^*} \max_{\substack{p \in e^*_{+}, p' \in e^*_{-}}} \frac{ \lVert p - p' \rVert_2 }{\text{cable radial thickness}} 
\end{align}
We define loose configurations as those with $L(E^*) > 0$ and \textbf{dense} configurations as those with $L(E^*) = 0$. A non-zero $L(E^*)$ occurs when the largest point pair $(p, p')$ distance between adjacent edges is non-zero, corresponding to a loose configuration. When $L(E^*) = 0$, at least two edges of the cable are directly adjacent, and segmentation for state estimation cannot reliably isolate overlapping edges. 

\subsection{Physical Experiments}
\subsubsection{Deploying HULK}
We discuss the design choices made when deploying HULK onto the dVRK.
\begin{figure}[H]
    \centering
    \includegraphics[width=1.0\linewidth]{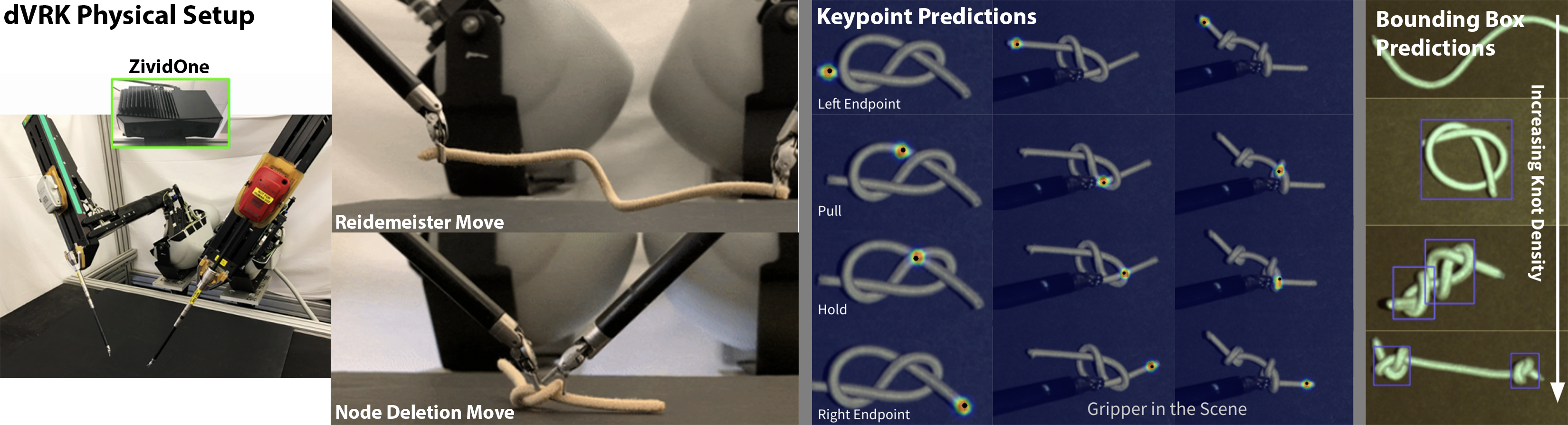}
    \caption{\textbf{HULK-G in Physical Experiments}: We deploy HULK-G onto a physical system using the dVRK robot for bilateral planning, the ZividOne RGB overhead camera, and keypoint and bounding box networks trained from real augmented data.}
    \label{fig:phys_exp_setup}
\end{figure}
\textbf{Manipulation:}
In order to avoid end-effector collisions that can result from purely top-down bilateral grasps, we perform Node Deletion moves with both the pull and pin grippers oriented inwards during the approach, \SI{30}{\degree} from the vertical (Figure  \ref{fig:phys_exp_setup} Node Deletion Move). For Node Deletion moves, the action vector $\vec{\hat{a}}$ is given by
$\hat{p}_\mathrm{pull} - \hat{p}_\mathrm{pin}$. We orient the pin gripper parallel to $\vec{\hat{a}}$ and the pull gripper orthogonal to $\vec{\hat{a}}$. Node Deletion moves are performed sequentially, with the pull grasp followed by the pin grasp in order to avoid shifting a knot during an action. For Reidemeister moves, we employ principal component analysis in a local crop of the hairtie mask around the left and right endpoints to grasp approximately orthogonal to the hairtie at each endpoint, as shown in the Reidemeister move of Figure \ref{fig:phys_exp_setup}. We use a simplified termination condition from Eq. \ref{eq:termination} shown below in Eq. \ref{eq:real_termination} as the stiffness of the hairtie elastic resulted in many false positives with the middle condition.
\begin{equation}
 \label{eq:real_termination}
      \underbrace{g(I) = \varnothing}_{\text{no knot detected}} \text { OR } \underbrace{t > T}_{\text{\# actions exceeded}}= 15 \times \# \text{ initial knots}.
 \end{equation}
\\ \textbf{Perception:} We train HULK's perception components (keypoints and bounding box networks) on real data since the level of supervision required is feasible to obtain and easy to augment. Additionally, this allows us to train on a broader set of cable states that are difficult to simulate due to differences in friction and elasticity between real and synthetic cables. In 350 real hair tie images, we hand-label the keypoints and knot bounding boxes and augment this dataset 20X using the affine and image-based transformations provided in Figure \ref{fig:aug_fig}. We use the networks trained on these datasets to perform Node Deletion and Reidemeister inference. One detail of performing Node Deletion moves sequentially is that we run inference once to predict and grasp the pull keypoint and a second inference with the end-effector in the scene to infer and grasp the pin point. Since planning is done in an open-loop fashion, the sequential inference allows the predicted pin grasp to be updated even if the pull grasp slightly jostles the cable initially. Lastly, before each action, we localize the hairtie within the workspace using color-based thresholding and use this adaptive crop when planning to keep the cable in frame during manipulation.

  \begin{figure}[!ht]
    \centering
    \includegraphics[width=10cm]{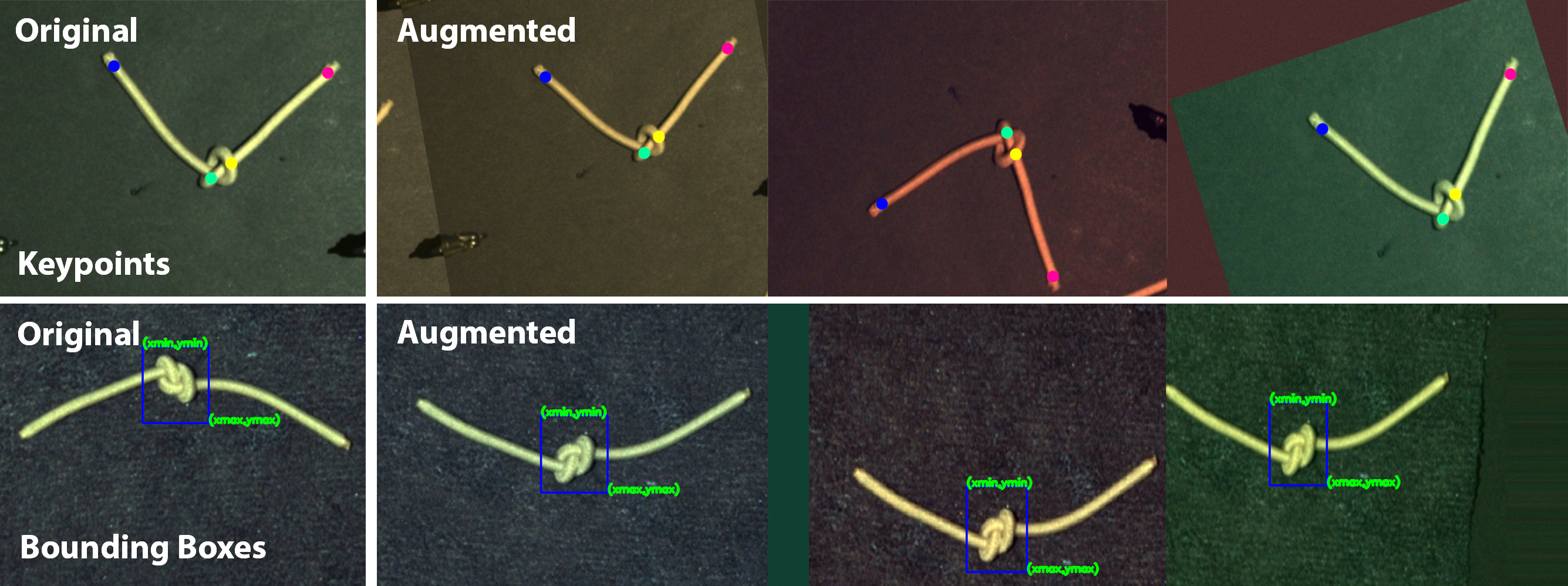}
    \qquad
    \begin{tabular}[b]{cc}\hline
      Augmentation Parameter & Amount \\ \hline
      Linear Contrast & (0.85,1.15) per channel \\
      Add & (-10,10) per channel \\
      Gamma Contrast & (0.9,1.1) \\
      Gaussian Blur & $\sigma=(0.0,0.6)$ \\
      Multiply Saturation & (0.95,1.05) \\
      Additive Gaussian Noise & scale=$(0,0.0125 \times 255)$ \\
      Vertical Flip & $p=0.5$ \\
      Horizontal Flip$^{**}$ & $p_\mathrm{kpt}=0.0, p_\mathrm{bbox}=0.5$ \\
      Scale & (0.8,1.2) \\
      Translate \% & (-20, 20) \\
      Rotate & (\SI{-20}{\degree}, \SI{20}{\degree}) \\
      Shear & (\SI{-10}{\degree}, \SI{10}{\degree}) 
      \\ \hline
    \end{tabular}
    \caption{\textbf{HULK-G Real Data Augmentations}: We augment a small set of image, annotation pairs for HULK-G's bounding box and keypoint modules using the affine and image transformations above. For keypoint augmentations, we disregard horizontal flips since $p_\mathrm{pull}, p_\mathrm{pin}$ are not recoverable under such transformations since we always perform untangling relative to the right endpoint. Some examples of augmented images and annotations are shown above.}
  \label{fig:aug_fig}
  \end{figure}
 
\subsubsection{Manipulation Details}
\begin{figure}[H]
    \centering
    \includegraphics[width=1.0\linewidth]{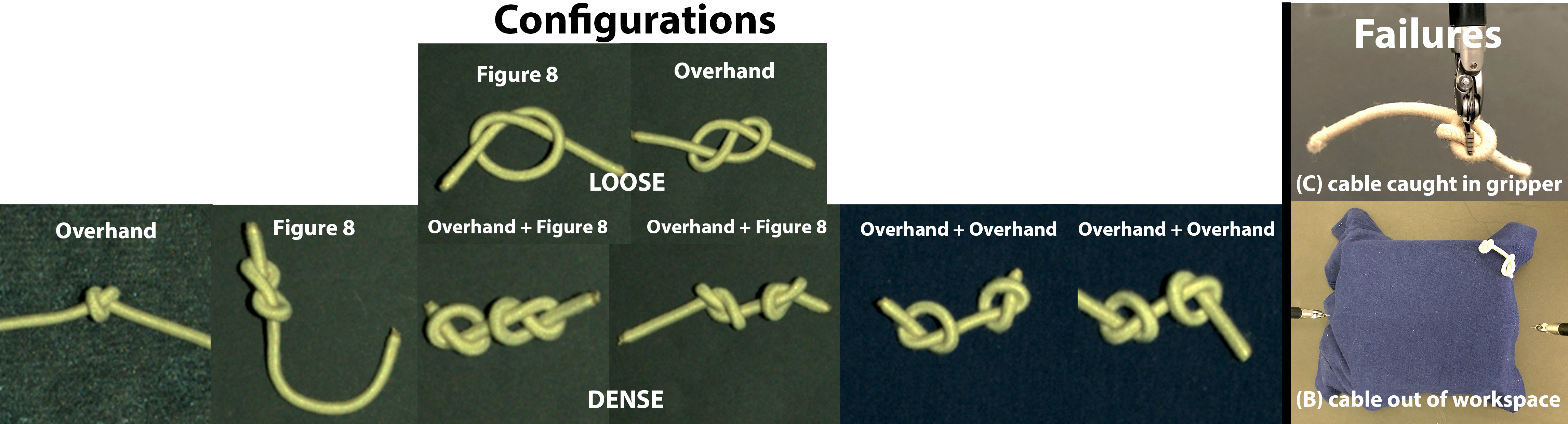}
    \caption{\textbf{Physical Experiment Configurations and Failure Modes}: We visualize loose and dense variants of Overhand and Figure-8 Knots used in physical experiments. Two common failure modes include the hair tie being caught in the dVRK end effector (Failure Mode C) due to cable friction/aggressive grasp planning and the hair tie leaving the reachable workspace due to consecutive poor predicted actions. }
    \label{fig:phys_exp_setup}
\end{figure}

 \begin{table}[!htbp]
\centering {
\begin{tabular}{m{0.25\linewidth}| m{0.75\linewidth}}
\textbf{Challenge} & \textbf{Workaround} \\ 
\hline
Pin/pull gripper collision & Purposely bias pin/pull keypoints apart during real data annotation, plan pin/pull grasp approaches at \SI{30}\degree from vertical to avoid collision \\
\hline
Cable out of workspace & Employ a third pick-place action (Recenter Move) to grasp at the detected rope mask center and drop at the global workspace center if the cable mask lies in an unreachable region \\
\hline 
Near misses due to knot density & Intentionally overextend the pull gripper 2mm lower than predicted pull point for more secure grasp 
 \end{tabular}}
 \caption{\textbf{Considerations for Physical Manipulation}: We elaborate on several strategies employed to mitigate the challenges of open-loop control in the setting of physical untangling with the dVRK. }
 \label{table:manip_failure_modes}
 \end{table}

\subsection{Simulation Experiments}
\subsubsection{Failure Modes}
We summarize the occurrence of failure modes across each experimental policy as discussed in Section{~\ref{sec:experiments}} in Table~\ref{table:failures}. The two prominent failure modes we observe in the simulated trials are bounding box false negatives and premature termination (Equation \ref{eq:termination}), corresponding to Failure Mode A. It is possible to recover from bounding box false negatives which delay untangling and, in the worst case, exceed the maximum allowed action count. On the other hand, premature termination immediately causes a failure to untangle the cable.  
\begin{table}[H]
\centering
 \begin{tabular}{||c || c || r || } 
 \hline
 Policy & Explanation & Count \\ 
 \hline\hline
Depth & bounding box false negative & 3 \\
\hline
Depth & premature termination (Equation \ref{eq:termination}) & 3 \\
\hline\hline
Random & bounding box false negative & 11 \\
\hline
Random & premature termination (Equation \ref{eq:termination}) & 47 \\
\hline\hline
HULK-G & bounding box false negative & 3 \\
\hline
HULK-G & premature termination (Equation \ref{eq:termination}) & 0 \\
\hline\hline
HULK-L & bounding box false negative & 1 \\
\hline
HULK-L & premature termination (Equation \ref{eq:termination}) & 5 \\
\hline
\end{tabular}
\caption{\textbf{Classification of Untangling Failure Causes in Simulator Experiments}}
\label{table:failures}
 \end{table}
 
The trials suggest that the Random policy experiences the highest occurance of failures both in bounding box false negatives and early triggers. This is likely because taking random actions produces particularly pathological configurations of the cable that are never seen during training of the bounding box module. HULK-L experiences the next most early triggers, likely caused by mispredicted bounding boxes that lead to erroneous pull and pin predictions that have a direct impact on the termination condition (Equation \ref{eq:termination}). Depth, HULK-G, and HULK-L all experience a comparable number of failures overall, but HULK-L and HULK-G show performance gains with empirically fewer actions. The occurrence of fewer false negatives and early triggers (compared to Random) indicates that HULK takes more sensible untangling actions that result in intermediate configurations closer to those seen in training the bounding box detector. These configurations more closely align with those produced by the algorithmic supervisor BRUCE. 
\end{document}